\documentclass[letterpaper, 10 pt, conference]{ieeeconf}  

\IEEEoverridecommandlockouts                              
\usepackage{amsmath,amssymb,amsfonts}
\usepackage{algorithm}
\usepackage{graphicx}
\usepackage{textcomp}
\usepackage{booktabs}
\usepackage{subcaption}
\usepackage[noend]{algpseudocode}
\usepackage{csquotes}
\usepackage{multirow}
\usepackage{tabularx}
\usepackage{soul}   
\usepackage{amsmath}    
\usepackage{bm}         
\usepackage[table]{xcolor}
\usepackage[hidelinks]{hyperref}
\hypersetup{
  colorlinks,
  citecolor= blue,
  linkcolor=black,
  urlcolor=black}
\usepackage[style=ieee, url=false, doi=false, natbib=true, mincitenames=1, maxcitenames=1, maxbibnames=99]{biblatex}
\newtheorem{theorem}{Theorem}

\title{\LARGE \bf
Optimal Transport-Guided Safety in Temporal Difference Reinforcement Learning
}

\author{Zahra Shahrooei and Ali Baheri
\thanks{Zahra Shahrooei and Ali Baheri are with the Department of Mechanical Engineering, Rochester Institute of Technology, Rochester, NY 14623 USA. (e-mail: zs9580@rit.edu; akbeme@rit.edu).}%
}

\begin{document}

\maketitle
\thispagestyle{empty}
\pagestyle{empty}

\begin{abstract}
The primary goal of reinforcement learning is to develop decision-making policies that prioritize optimal performance, frequently without considering safety. In contrast, safe reinforcement learning seeks to reduce or avoid unsafe behavior. This paper views safety as taking actions with more predictable consequences under environment stochasticity and introduces a temporal difference algorithm that uses optimal transport theory to quantify the uncertainty associated with actions. By integrating this uncertainty score into the decision-making objective, the agent is encouraged to favor actions with more predictable outcomes. We theoretically prove that our algorithm leads to a reduction in the probability of visiting unsafe states. We evaluate the proposed algorithm on several case studies in the presence of various forms of environment uncertainty. The results demonstrate that our method not only provides safer behavior but also maintains the performance. A Python implementation of our algorithm is available at \href{https://github.com/SAILRIT/Risk-averse-TD-Learning}{https://github.com/SAILRIT/OT-guided-TD-Learning}.

\end{abstract}


\maketitle

\section{INTRODUCTION}\label{sec:introduction}

Safe reinforcement learning (RL) algorithms are essential for many real-world applications in robotics \cite{richter2019open}, autonomous systems \cite{kendall2019learning}, finance \cite{hambly2023recent}, and healthcare \cite{gu2024review}. Various approaches have been explored to integrate safety into RL. A comprehensive review of these methods can be found in \cite{garcia2015comprehensive}.

Several early safe RL studies incorporate safety into the optimization criterion \cite{heger1994consideration,gaskett2003reinforcement,sato2001td,geibel2005risk,achiam2017constrained,chow2019lyapunov}. For example, in the worst-case criterion, a policy is considered optimal if it maximizes the worst-case return, which reduces variability due to inherent or parameter uncertainty \cite{heger1994consideration,gaskett2003reinforcement}. The optimization criterion can also be adjusted to balance return and a risk metric, such as a linear combination of return and the variance of return \cite{sato2001td} or as the probability of entering an error state \cite{geibel2005risk}. The other way is to optimize the return subject to constraints, resulting in the constrained optimization criterion \cite{achiam2017constrained,chow2019lyapunov,brunke2022safe}. Other approaches aim to avoid heuristic exploration strategies, which are blind to the safety of actions. Instead, they propose modifications to the exploration process to guide the agent toward safer regions. Safe exploration techniques include prior knowledge of the task for search initialization \cite{okawa2022safe}, learn from human demonstrations \cite{dai2023safe}, and incorporate a risk metric into the algorithm \cite{gehring2013smart,law2005risk}.

In the latter, uncertainty estimates are commonly used to direct exploration and enhance policy stability. To accomplish this, a method is to estimate the uncertainty of future returns. For example, \citet{o2018uncertainty} propose an uncertainty Bellman equation to approximate the variance of Q-value posterior distributions. Distributional RL studies \cite{malekzadeh2023unified, dabney2018distributional, mavrin2019distributional, dabney2018distributional, zhou2021non, tang2018exploration} explicitly estimate the distribution of returns and measure aleatory uncertainty, which originates from intrinsic stochasticity in the environment and is irreducible. Our work differs from these works in that we do not need the return distribution to guide the agent toward taking safer actions. There are other works that use information-directed sampling
to prevent the negative consequences of aleatory uncertainty by
making a trade-off between instantaneous regret and information gain \cite{chen2022efficient,nikolov2018information}. However, these information-directed approaches require learning the transition dynamics of the environment. An alternative approach is to assess the uncertainty in the model dynamics rather than concentrating on the uncertainty of the returns. For instance, \citet{yu2020mopo} estimate the uncertainty of the dynamics model and penalize actions that result in uncertain returns. Our method differs from these approaches because it does not require any model or prior knowledge of the dynamics to direct the exploration.

In this work, we define safety in stochastic environments as choosing actions with more predictable consequences. We explore the use of optimal transport (OT) theory to quantify the uncertainty level corresponding to different actions. OT is highly valued for its ability to measure and optimize the alignment between probability distributions \cite{santambrogio2015optimal} by minimizing the cost of transforming one distribution into another while taking to account the geometry of the distributions as shown in Fig. \ref{OT}. There are a few applications of OT in safe RL \cite{baheri2023risk,baheri2023understanding,queeney2023optimal,metelli2019propagating,shahrooei2024optimal,baheri2025wasserstein,nkhumise2025studying}. For example, \citet{metelli2019propagating} propose a novel approach called Wasserstein Q-learning (WQL), which uses Bayesian posterior distributions and Wasserstein barycenters to model and propagate uncertainty in RL. Their method demonstrates improved exploration and faster learning in tabular domains compared to classic RL algorithms. \citet{baheri2023risk} uses OT for reward shaping in Q-learning. Through the minimization of the Wasserstein distance between the policy’s stationary distribution and a predefined safety distribution over state space, the agent is encouraged to visit safe states more frequently.

\begin{figure}[!t]
    \centering
     \includegraphics[scale=0.45]{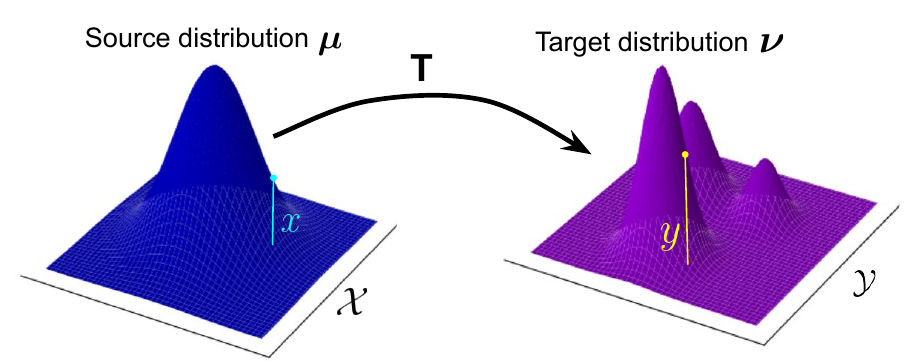}
     \caption{Conceptual illustration of optimal transport theory. Here, $\mu(x)$ is a probability distribution over the source space $\mathcal{X}$, and $\nu(y)$ is a probability distribution over the target space $\mathcal{Y}$. The arrows represent the optimal transport plan $T$, which reallocates mass from $\mu$ to $\nu$ to minimize the total transport $\operatorname{cost} c(x, y)$.}
     \label{OT}
\end{figure}

In this paper, we propose an OT-based temporal difference (TD) framework that considers both the reward and the uncertainty associated with actions without relying on expert knowledge or predefined safety constraints. We use Wasserstein distance to quantify the total uncertainty score at each state and prioritize the actions that contribute less to this score. This encourages the agent to take less uncertain actions more frequently, which leads to safer overall behavior.

The contributions of this paper are: (i) introduction of an OT-guided TD learning algorithm to enhance agent safety by integrating action uncertainty into decision-making in addition to the reward signal without requiring prior expert knowledge or safety constraints, (ii) deriving safety bounds of the algorithm which demonstrates less visitation to unsafe states (Theorem \ref{thm:safety_bounds}), and (iii) applications on various case studies under different forms of uncertainty in reward function, transition function, and states to show our algorithm reduces visiting unsafe states while preserving performance.

We structure this paper as follows: Section \ref{sec:prelim} provides a mathematical overview of temporal difference learning and optimal transport theory. Section \ref{sec:SafeSarsa} elaborates on incorporating optimal transport-based action uncertainty into temporal difference learning. Section \ref{sec:Evaluation} presents experiments on various environments and their results. Section \ref{sec: Conclusion} concludes and suggests future direction.

\section{PRELIMINARIES}
\label{sec:prelim}

This section reviews Markov decision processes, partially observable Markov decision processes, temporal difference learning algorithms, and the basic principles of OT theory.

\subsection{MARKOV DECISION PROCESSES AND TEMPORAL DIFFERENCE LEARNING}

\noindent{\textbf{Markov Decision Processes (MDPs). }}MDPs \cite{sutton1998reinforcement} represent a fully observable RL environment. A finite MDP is defined by the tuple $\mathcal{M} = (\mathcal{S}, \mathcal{A}, \mathcal{T},\mathcal{R}, \gamma)$, where $\mathcal{S}$ is a set of states, $\mathcal{A}$ is a set of actions, $\mathcal{T}: \mathcal{S} \times \mathcal{A} \times \mathcal{S} \to [0, 1]$ 
is the transition probability function, where $\mathcal{T}(s'|s, a)$ represents the probability of transitioning to state $s'$ when action $a$ is taken in state $s$, $\mathcal{R}: \mathcal{S} \times \mathcal{A} \to \mathbb{R}$ is the reward function where \(\mathcal{R}(s, a) = r\) represents the immediate reward received for taking action \(a\) in state \(s\), and \(\gamma \in [0, 1)\) is the discount factor. The agent follows a policy $\pi : \mathcal{S} \times \mathcal{A} \rightarrow [0, 1]$ that maps states to action probabilities, with the objective to maximize the expected discounted return, $G_t = \sum_{k=0}^{\infty} \gamma^k r_{t+k+1}$.

\noindent{\textbf{Temporal Difference Learning. }}The Q-value of action $a$ in state $s$ under policy $\pi$ is given by $Q_{\pi}(s, a) = \mathbb{E}_{\pi}[G_t | s_t = s, a_t = a]$, which can be incrementally learned. In one-step TD, the Q-value update rule is $Q(s_t, a_t) \leftarrow Q(s_t, a_t) + \alpha \delta_t$ where $\delta_t$ is the TD error at time $t$ and $\alpha$ is the learning rate. In Q-learning algorithm \cite{watkins1992q}, the TD error is defined as:
\begin{equation}
\delta_t = r_{t+1} + \gamma \max_{a} Q(s_{t+1}, a) - Q(s_t, a_t)
\label{eq:4} 
\end{equation}
Similarly, in SARSA algorithm \cite{sutton1998reinforcement}, the TD error is given by: 
\begin{equation}
\delta_t = r_{t+1} + \gamma Q(s_{t+1}, a_{t+1}) - Q(s_t, a_t).
\label{eq:3} 
\end{equation}
SARSA($\lambda$) extends one-step SARSA using eligibility traces to incorporate multi-step updates and improve learning efficiency \cite{sutton1998reinforcement}. An eligibility trace tracks the degree to which each state-action pair has been recently visited, which enables updates to consider a weighted history of past experiences. The Q-value update rule in SARSA($\lambda$) is:
\begin{equation}
Q_{t+1}(s, a) = Q_t(s, a) + \alpha \delta_t e_t(s, a)
\end{equation}
where $e_t(s, a)$ is the eligibility trace. The eligibility trace $e_t(s, a)$ decays over time and is updated as follows:
\begin{equation}
e_t(s, a) = 
\begin{cases} 
\gamma \lambda e_{t-1}(s, a) + 1 & \text{if } s = s_t \text{ and } a = a_t \\
\gamma \lambda e_{t-1}(s, a) & \text{otherwise}
\end{cases}
\end{equation}
where $\lambda \in [0, 1]$ controls the trace decay rate. TD algorithms often use an $\epsilon$-greedy strategy for action generation. The parameter $\epsilon$ can either be fixed or decay over time to balance exploration and exploitation. The behavioral policy is expressed as:

\begin{equation}
\pi(s_t, a_t) =
\begin{cases}
1 - \epsilon + \dfrac{\epsilon}{|\mathcal{A}|} & \text{if } a_t = \arg\max_{a'_t} Q(s_t, a'_t) \\
\dfrac{\epsilon}{|\mathcal{A}|} & \text{otherwise}
\end{cases}
\label{Eq:EPG}
\end{equation}

\noindent{\textbf{Partially Observable MDPs (POMDPs). }}POMDPs \cite{krishnamurthy2016partially} generalize MDPs to account for environments where the agent cannot directly observe the underlying state. A POMDP is defined by the tuple $\mathcal{PM} = (\mathcal{S}, \mathcal{A}, \mathcal{O}, \mathcal{T}, \mathcal{R}, \mathcal{Z}, \gamma)$, where $\mathcal{S}$, $\mathcal{A}$, $\mathcal{T}$, $\mathcal{R}$, and $\gamma$ maintain their definitions from MDPs, and $\mathcal{O}$ is a set of observations the agent can perceive. The observation function $\mathcal{Z}: \mathcal{S} \times \mathcal{A} \to \mathcal{O}$ maps states and actions to observation probabilities.

\subsection{OPTIMAL TRANSPORT THEORY}

The OT theory aims to find minimal-cost transport plans to move one probability distribution to another within a metric space. This involves a cost function $c(x, y)$ and two probability distributions, $\mu(x)$ and $\nu(y)$. The goal is to find a transport plan that minimizes the cost of moving $\mu$ to $\nu$ under $c(x, y)$, often using the Euclidean distance for explicit solutions \cite{villani2009optimal}.

We focus on discrete OT theory, assuming $\mu$ and $\nu$ as source and target distributions, respectively, both belonging to $\mathcal{P}p(\mathbb{R}^n)$, i.e., they are probability measures with finite \(p\)-th moments \cite{villani2009optimal} with finite supports $\{x_i\}_{i=1}^{m_1}$ and $\{y_j\}_{j=1}^{m_2}$, and corresponding probability masses $\{a_i\}_{i=1}^{m_1}$ and $\{b_j\}_{j=1}^{m_2}$. The cost between support points is represented by an $m_1 \times m_2$ matrix $C$, where $C_{ij} = |x_i - y_j|_p^p$ denotes the transport cost from $x_i$ to $y_j$. The OT problem seeks the transport plan $P^\ast$ that minimizes the cost while ensuring that the marginals of $P^\ast$ match $\mu$ and $\nu$:

\begin{equation}
\min_{P \in \mathbb{R}^{m_1 \times m_2}} \sum_{i=1}^{m_1} \sum_{j=1}^{m_2} P_{ij} C_{ij}  
\label{eq:3}
\end{equation}
Here, the coupling matrix \( P_{ij} \) indicates the probability mass transported from \( x_i \) to \( y_j \) and \( \sum_{j=1}^{m_2} P_{ij} = a_i \)
for all \( i \), and \( \sum_{i=1}^{m_1} P_{ij} = b_j\) for all \( j \). Additionally, \( P_{ij} \geq 0 \) for all \( i, j \).

\textit{Definition 1 \cite{villani2009optimal}:} 
The Wasserstein distance between $\mu$ and $\nu$ is computed using the OT plan $P^\ast$ obtained from solving the above linear programming problem:
\begin{equation}
W_p(\mu, \nu) = \left( \langle P^\ast, C \rangle \right)^{\frac{1}{p}}
\label{eq:4}
\end{equation}
where $ p\geq 1$ and $\langle \cdot, \cdot \rangle$ denotes the inner product.

To enhance numerical stability and computational efficiency, an entropy regularization term can be added to the objective, leading to the regularized Wasserstein distance, which can be solved iteratively using the Sinkhorn iterations \cite{cuturi2013sinkhorn}.

\textit{Definition 2 \cite{cuturi2013sinkhorn}:} The Entropy-regularized OT problem is formulated as:
\begin{equation}
\min_{P \in \mathbb{R}^{m_1 \times m_2}} \quad \sum_{i=1}^{m_1} \sum_{j=1}^{m_2} P_{ij} C_{ij} + \varepsilon \sum_{i=1}^{m_1} \sum_{j=1}^{m_2} P_{ij} (\log P_{ij} - 1)
\label{eq:5}
\end{equation}
where $\varepsilon > 0$ is a regularization parameter that balances transport cost and the entropy of the transport plan $P$.

\section{METHODOLOGY}
\label{sec:SafeSarsa}

Consider an RL agent interacting with an environment defined by a discrete MDP $\mathcal{M} = (\mathcal{S}, \mathcal{A}, \mathcal{T},\mathcal{R}, \gamma)$, where a subset of states \(\mathcal{S}_{\text{u}} \subset \mathcal{S}\), denoted as unsafe states, exhibits stochasticity in the reward function \(\mathcal{R}\) or transition function \(\mathcal{T}\). These unsafe states are undesirable due to their unpredictable outcomes, and the agent aims to visit them less frequently to enhance safety. For each state $s$, we define the \emph{Q-distribution} $Q_s$ as the normalized distribution over the agent’s estimated Q-values for the available actions $a_{i}$, $i = 1, \ldots, N$. Formally, we have

\begin{equation}
Q_s = \sum_{i=1}^{N} q_i^s \delta_{\mathbf{a}_i}
\end{equation}
where $q_i^s$ is the probability assigned to action $a_i$ based on the current Q-value estimations, and $\delta_{\mathbf{a}_i^s}$ is the Dirac measure centered at $a_i$. Intuitively, $Q_s$ captures how the agent’s Q-values are distributed across actions at state $s$. We also introduce a corresponding T-distribution \( T_t \), which is constructed similarly to the Q-distribution but is derived from the target values used in SARSA update rule. Specifically, for each state \( s \) and each action \( a_i \) in \( s \), we maintain a target value in a buffer, which is updated whenever \( a_i \) is selected in \( s \). When the agent chooses action \( a_t \) in state \( s_t \), the target value for \( (s_t, a_t) \) is updated to $r_{t+1} + \gamma Q(s_{t+1}, a_{t+1})$, where \( r_{t+1} \) is the observed reward and \( \gamma \) is the discount factor. For actions \( a_i \neq a_t \) in state \( s_t \), the target value preserves the value from the most recent instance in which \( a_i \) was selected in \( s_t \). Hence, we have 

\begin{equation}
\quad T_t = \sum_{i=1}^{N} p_i^t \delta_{\mathbf{a}_i}
\end{equation}
where $p_i^t$ is the probability associated with action $a_i$ based on target values.

For each action $a_i$ in state $s$, we define an uncertainty score $U(s, a_i)$. The goal is to quantify how much an action contributes to the overall uncertainty score of the agent’s policy in that state. First, we compute the OT map $P^\ast$ between $Q_s$ and $T_t$ by solving the entropy-regularized Wasserstein distance formulation. We note that here, the cost matrix $C(a_{i}, a_{j}), \text{ for } i,j = 1, \dots, N$ is inherently task-dependent and its definition varies depending on the structure of the action space. The total uncertainty score in state $s$ is measured by the Wasserstein distance $W(Q_s, T_t)$. For an action $a_i$, the flow $\Delta(s, a_i)$ is defined as the absolute difference between the outgoing flow (transport from $a_i$ to other actions $b \neq a_i$) and the incoming flow (transport from other actions $b \neq a_i$ to $a_i$):
\begin{equation}
\Delta(s, a_i) = \left| \sum_{b \neq a_i} P^\ast_{a_i, b} - \sum_{b \neq a_i} P^\ast_{b, a_i} \right|
\label{eq:12}
\end{equation}
Here, the first summation $\sum_{b \neq a_i} P^\ast_{a_i, b}$ represents how much probability mass is transported away from $a_i$ to other actions, while the second summation $\sum_{b \neq a_i} P^\ast_{b, a_i}$ denotes how much mass flows into $a_i$ from other actions. The absolute difference between these two flows, $\Delta(s, a_i)$, captures how much the probability of an action in $Q_s$ must be \enquote{redistributed} to match $T_t$. We subsequently normalize this value by the total uncertainty in state $s$:
\begin{equation}
U(s, a_i) = \frac{\Delta(s,a_i)}{W(Q_s, T_t)} \label{eq:13}
\end{equation}
\begin{figure}[!b]
    \centering
     \includegraphics[scale=0.3]{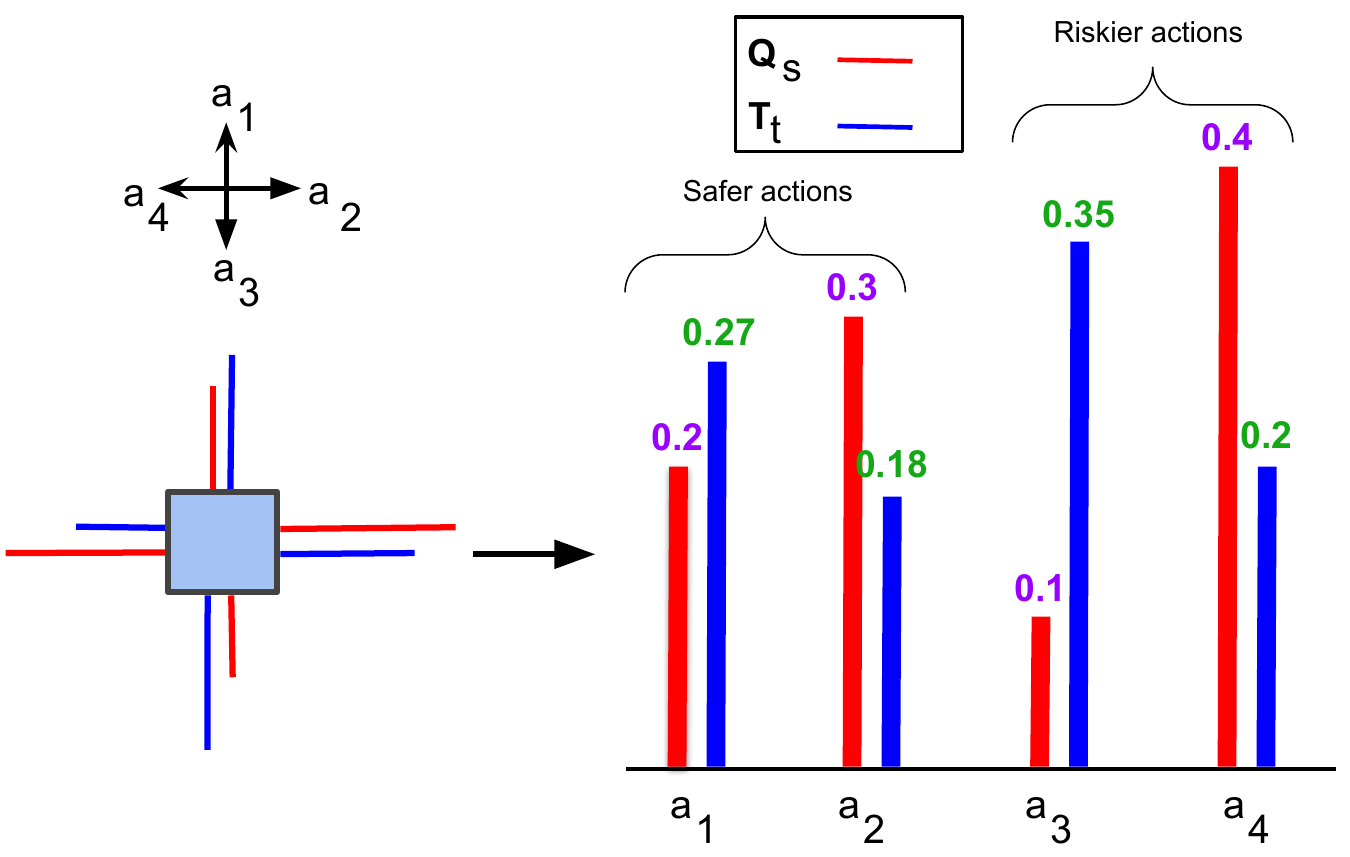}
     \caption{Example: For a fixed state and four available actions, we compute the uncertainty score of the state $W(Q_s, T_t)$ and the contribution of each action to this uncertainty score.}
     \label{OT2}
\end{figure}

\begin{table}[!b]
    \centering
    \caption{Example: Incorporating uncertainty into the behavioral policy considering $\beta=0.5$. Standard SARSA chooses $a_{4}$, while our algorithm prefers $a_{2}$. The safest available action is $a_{1}$.}
    \label{tab:example}
    \renewcommand{\arraystretch}{1.5} 
    \begin{tabular}{@{}lccc@{}}
        \toprule
        \textbf{Action} & \boldmath{$Q(s,a)$} & \boldmath{$U(s,a)$} & \boldmath{$Q(s,a)-\beta U(s,a)$} \\
        \midrule
        $a_{1}$ & $0.2$ & \boldmath{\hl{$0.21$}} & $0.09$ \\
        $a_{2}$ & $0.3$ & $0.37$ & \boldmath{\hl{$0.112$}} \\
        $a_{3}$ & $0.1$ & $0.78$ & $-0.29$ \\
        $a_{4}$ & \boldmath{\hl{$0.4$}} & $0.62$ & $0.087$ \\
        \bottomrule
    \end{tabular}
    \renewcommand{\arraystretch}{1.0} 
\end{table}

The uncertainty score $U(s, a_i)$ thus reflects how much action $a_i$ is responsible for the mismatch between $Q_s$ and $T_t$. Equivalently, it reveals to what extent the action $a_i$ needs to be adjusted for $Q_s$ to align with $T_t$ in a cost-efficient manner. Higher values of $\Delta(s, a_i)$ indicate larger corrections to the probability of $a_i$ in $Q_s$, which suggests more unpredictable outcomes. While standard SARSA does not account for safety, we incorporate the above uncertainty score into the behavioral policy. Let $\beta$ be a sensitivity coefficient that specifies how strongly the agent prioritizes safety relative to reward. We then modify the $\epsilon$-greedy behavioral policy $\pi^b$ as follows:

{\small
\begin{equation}
\pi^b(a \mid s) =
\begin{cases}
1 - \epsilon + \dfrac{\epsilon}{|\mathcal{A}|} & \text{if } a = \arg\max_{a'} [Q(s, a') - \beta U(s, a')] \\
\dfrac{\epsilon}{|\mathcal{A}|} & \text{otherwise}.
\end{cases}
\label{eq:14}
\end{equation}
}
where the $\epsilon$ component guarantees infinite exploration of all actions in visited states over time. Actions that lead to unsafe states typically have higher uncertainty scores $U(s, a)$ due to the stochasticity of those states, which increases the flow imbalance $\Delta(s, a)$. This property enables the algorithm to prioritize safer actions and reduce the likelihood of visiting $\mathcal{S}_{u}$. In other words, the agent’s action-selection process is biased to favor actions with both high Q-values and low uncertainty scores. Indeed, we encourage the agent to choose the action for which it has the highest confidence in the outcome among the actions experienced in that state. This approach enables a directed exploration strategy that prioritizes safer actions with higher rewards. Example $1$ further illustrates the influence of this uncertainty score on both Q-values and policy decisions. In particular, it demonstrates the trade-off between exploiting high-return actions and mitigating highly uncertain actions to ensure safer exploration and more stable learning.

In the context of POMDPs, we use SARSA($\lambda$) variant. Specifically, $Q_s$ and $T_t$ distributions associated with the agent's current observation $o$ and the available actions $a_{1},\cdots ,a_{N}$ are integrated into the uncertainty score term $U(o, a_i)$. Moreover, the flexibility of our approach enables its extension to scenarios that involve using multiple consecutive observations $o_{t-n},\cdots ,o_{t}$, which can help the algorithm better capture non-Markovian properties.

\textit{Example: }Consider a fixed state $s$ with four available actions. For this state, we have access to $Q_s$ and $T_t$ distributions over actions as depicted in Fig. \ref{OT2}. Table. \ref{tab:example} shows how the uncertainty score term affects the Q-values and action selection. As observed, $a_1$ is the safest action and $a_{4}$ offers the highest reward. In decision-making, standard SARSA prefers action $a_{4}$, whereas OT-guided SARSA chooses $a_{2}$ to balance the reward and safety.


\begin{theorem}\label{thm:safety_bounds}Given a finite MDP $\mathcal{M} = (\mathcal{S}, \mathcal{A}, \mathcal{T}, \mathcal{R}, \gamma)$ with state space $\mathcal{S}$ and action space $\mathcal{A}$, let $\mathcal{S}_{\text{u}} \subset \mathcal{S}$ be a set of unsafe states to avoid. We define $\pi^0$ as the $\epsilon$-greedy policy with respect to $Q(s, a)$, and $\pi_t^b$ as the $\epsilon$-greedy policy with respect to $Q(s, a) - \beta U(s, a)$ as in \ref{eq:14}. For each $s \in \mathcal{S}$, define the set of unsafe actions:
\[
\mathcal{A}_{\text{u}}(s) = \left\{ a \in \mathcal{A} : \mathcal{T}(s' \in \mathcal{S}_{\text{u}} \mid s, a) \geq \delta \right\}
\]
where $\mathcal{T}(s' \in \mathcal{S}_{\text{u}} \mid s, a) = \sum_{s' \in \mathcal{S}_{\text{u}}} \mathcal{T}(s' \mid s, a)$, and $\delta > 0$ is a threshold. Assume the Markov chains induced by the limiting policies $\pi^0$ and $\pi^b = \lim_{t \to \infty} \pi_t^b$ are irreducible and aperiodic. For sufficiently large $\beta$ and under the conditions:
\begin{enumerate}
    \item $\sum_t \alpha_t(s, a) = \infty$,
    \item $\sum_t \alpha_t(s, a)^2 < \infty$,
    \item $\mathcal{R}(s, a)$ is bounded,
\end{enumerate}
there exists a constant $0 < c < 1$ such that:
\begin{equation}
\lim_{t \to \infty} \Pr_{\pi_t^b}\left[s \in \mathcal{S}_{\text{u}}\right]
\leq  c\cdot \Pr_{\pi^0}\left[s \in \mathcal{S}_{\text{u}}\right].
\end{equation}
\end{theorem}
\begin{proof}
Since $\mathcal{S}$ and $\mathcal{A}$ are finite, any stationary policy $\pi : \mathcal{S} \to \mathcal{P}(\mathcal{A})$ ($\mathcal{P}(\mathcal{A})$ represents the space of probability distributions over $\mathcal{A}$) induces a Markov chain with transition matrix:
\[
P_{\pi}(s' \mid s) = \sum_{a \in \mathcal{A}} \pi(a \mid s) \mathcal{T}(s' \mid s, a).
\]
Given that the Markov chains under $\pi^0$ and $\pi^b$ are irreducible and aperiodic, they each have unique stationary distributions $\mu_{\pi^0}$ and $\mu_{\pi^b}$, respectively, satisfying:
\[
\mu_{\pi}(s') = \sum_{s \in \mathcal{S}} \mu_{\pi}(s) P_{\pi}(s' \mid s), \quad
\sum_{s \in \mathcal{S}} \mu_{\pi}(s) = 1.
\]
The long-run probability of being in $\mathcal{S}_{\text{u}}$ under policy $\pi$ is:
\[
\Pr_{\pi}\left[s \in \mathcal{S}_{\text{u}}\right] = \mu_{\pi}(\mathcal{S}_{\text{u}}) = \sum_{s \in \mathcal{S}_{\text{u}}} \mu_{\pi}(s).
\]
Under the given convergence conditions $\pi_t^b \to \pi^b$, hence:
\[
\lim_{t \to \infty} \Pr_{\pi_t^b}\left[s \in \mathcal{S}_{\text{u}}\right] = \mu_{\pi^b}(\mathcal{S}_{\text{u}}).
\]
Let us define the one-step probability of entering an unsafe state from state $s$ under policy $\pi$:
\[
p_{\pi}(s) = \mathcal{T}(s' \in \mathcal{S}_{\text{u}} \mid s, \pi) = \sum_{a \in \mathcal{A}} \pi(a \mid s) \mathcal{T}(s' \in \mathcal{S}_{\text{u}} \mid s, a).
\]
For $\pi^0$ and $\pi^b$, the $\epsilon$-greedy policy is \ref{Eq:EPG} and \ref{eq:14}, respectively. We assume that $U(s, a) \geq \eta > 0$ for $a \in \mathcal{A}_{\text{u}}(s)$ and $U(s, a) < \eta$ for $a \notin \mathcal{A}_{\text{u}}(s)$, reflecting that unsafe actions have higher uncertainty. For sufficiently large $\beta$, if there exists $a_s \notin \mathcal{A}_{\text{u}}(s)$ with competitive $Q(s, a_s)$, then:
\[
Q(s, a_s) - \beta U(s, a_s) > Q(s, a) - \beta U(s, a), \quad \forall a \in \mathcal{A}_{\text{u}}(s)
\]
Since $\beta U(s, a) \geq \beta \eta$ is large for unsafe actions, the greedy action under $\pi^b$ satisfies 
$a^* = \arg\max_{a'} [Q(s, a') - \beta U(s, a')] \notin \mathcal{A}_{\text{u}}(s)$. Under $\pi^b$, if $\beta$ is large, since $\mathcal{T}(s' \in \mathcal{S}_{\text{u}} \mid s, a^*) < \delta$, and $\pi^0$ selects $a \in \mathcal{A}_{\text{u}}(s)$ with higher probability, typically $p_{\pi^b}(s) < p_{\pi^0}(s)$. Then, we define $A = \mathcal{S}_{\text{u}},\ B = \mathcal{S} \setminus \mathcal{S}_{\text{u}}$. The flow balance in steady state gives:
\[
\sum_{s \in B} \mu_{\pi}(s) p_{\pi}(s) = \sum_{s \in A} \mu_{\pi}(s) [1 - p_{\pi}(s)].
\]
Let $\alpha_{\pi} = \sum_{s \in B} \mu_{\pi}(s) p_{\pi}(s)$, the flow into $A$. If $p_{\pi^b}(s) < p_{\pi^0}(s)$ for $s \in B$, and assuming $p_{\pi^b}(s) \approx p_{\pi^0}(s)$ for $s \in A$ (similar behavior in unsafe states), then $\alpha_{\pi^b} < \alpha_{\pi^0}$. Suppose there exists $\gamma > 0$ such that $1 - p_{\pi}(s) \geq \gamma$ for $s \in A$. Then:
\[
\mu_{\pi}(A)[1 - p_{\pi}(s)]_{\text{avg}} = \alpha_{\pi}, \quad \mu_{\pi}(A) \leq \frac{\alpha_{\pi}}{\gamma}.
\]
Therefore, $\mu_{\pi^b}(A) \leq \frac{\alpha_{\pi^b}}{\gamma} < \frac{\alpha_{\pi^0}}{\gamma} \leq \mu_{\pi^0}(A)$. Let us define $c = \frac{\alpha_{\pi^b}}{\alpha_{\pi^0}}$, where $0 < c < 1$. Since $\alpha_{\pi^b} < \alpha_{\pi^0}$, we have: 
\[
\lim_{t \to \infty} \Pr_{\pi_t^b}\left[s \in \mathcal{S}_{\text{u}}\right]
= \mu_{\pi^b}(\mathcal{S}_{\text{u}})
\leq c \cdot \mu_{\pi^0}(\mathcal{S}_{\text{u}})
= c \cdot \Pr_{\pi^0}\left[s \in \mathcal{S}_{\text{u}}\right]
\]
and the proof is complete.
\end{proof}


\section{EXPERIMENTS}
\label{sec:Evaluation}

We evaluate OT-guided SARSA in three case studies with different sources of environment uncertainty, which can influence the decision-making process. In particular, we focus on stochastic reward function, stochastic transition dynamics, and stochastic observation function (Fig. \ref{fig:casestudies}). For all case studies, the cost matrix used for computing the uncertainty score is \(C_{ij} = 1\) if \(i \neq j\), and \(0\) otherwise. Furthermore, $\gamma=0.99$ and $\epsilon=0.1$ for all case studies.


    
    
    
    

\begin{figure}[!h]
    \centering
    \begin{subfigure}{0.4\textwidth}
        \centering
        \includegraphics[scale=0.4]{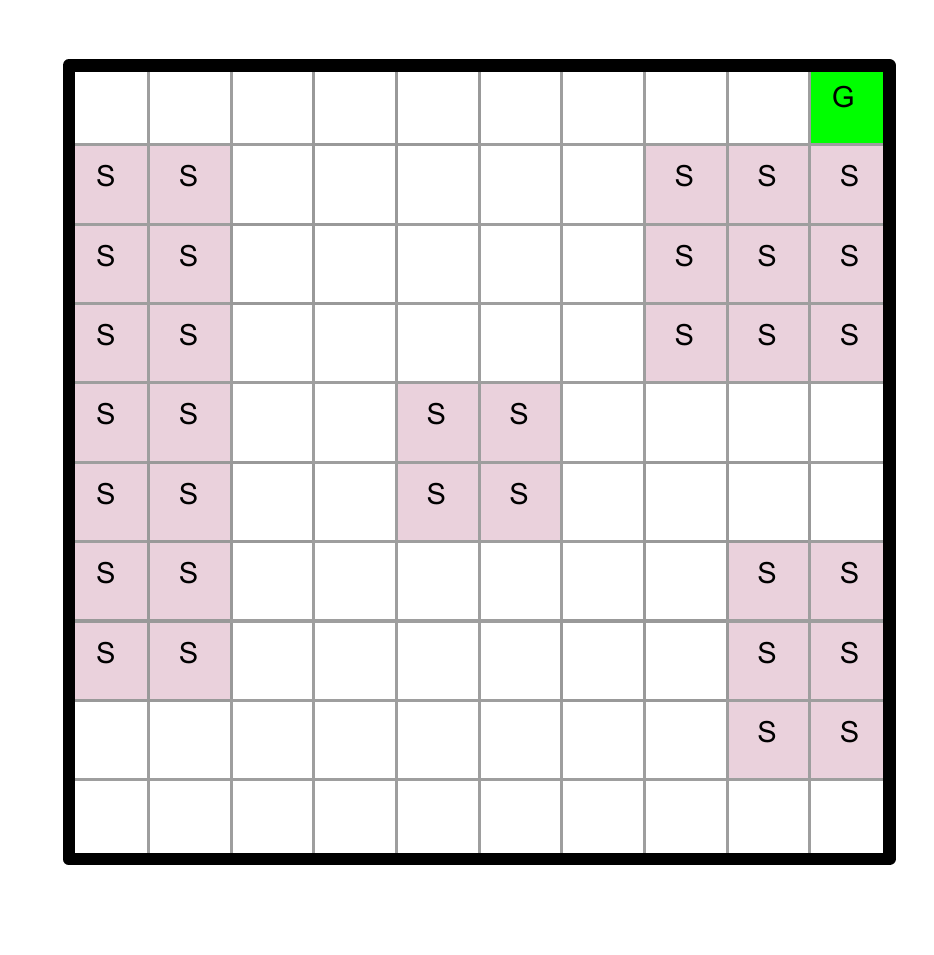}
        \caption{Grid-world environment}
        \label{fig:gridworld}
    \end{subfigure}
    \begin{subfigure}{0.4\textwidth}
        \centering
        \includegraphics[scale=0.4]{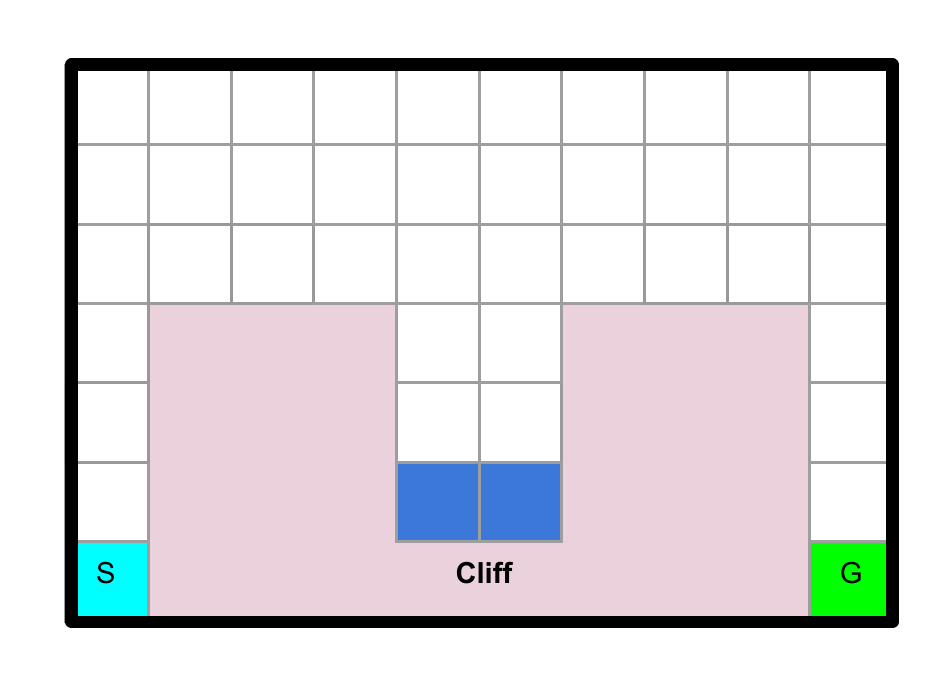}
        \caption{Cliff walking case study}
        \label{fig:cliffwalk}
    \end{subfigure}
    
    
    \begin{subfigure}{0.4\textwidth}
        \centering
        \includegraphics[scale=0.4]{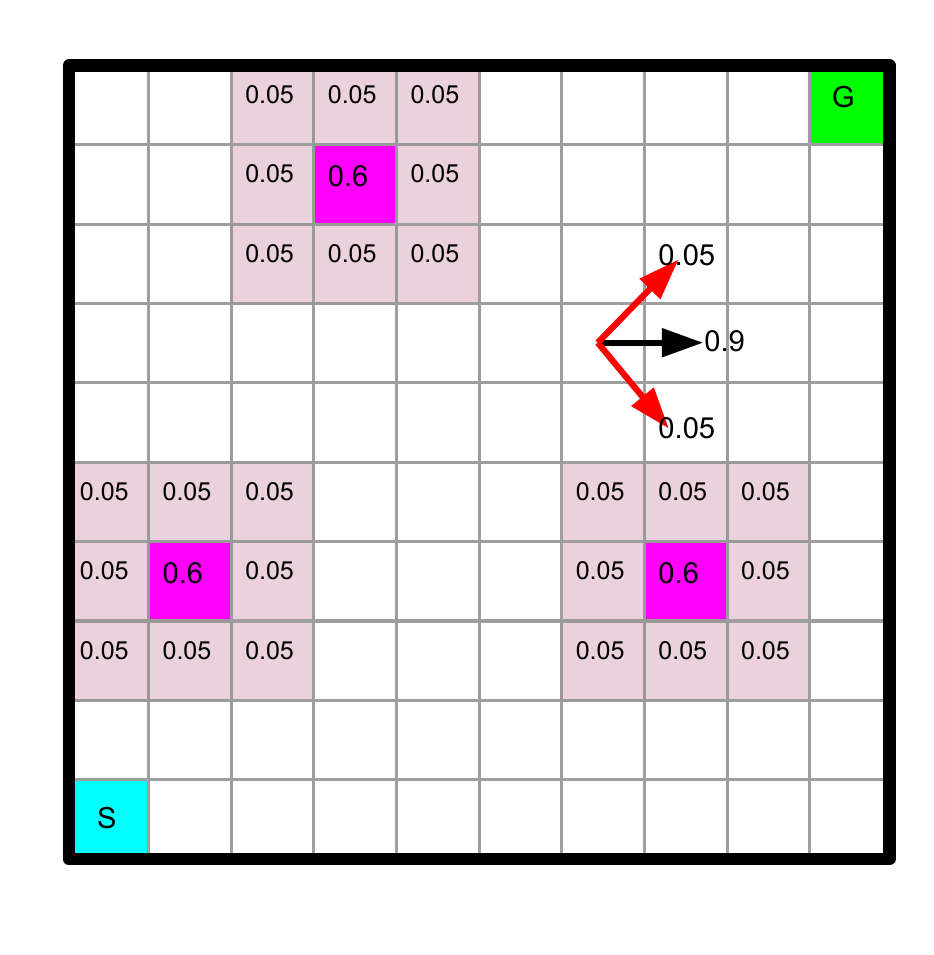}
        \caption{Rover navigation task}
        \label{fig:rover}
    \end{subfigure}
    
    \caption{Case studies: Grid-world featuring reward uncertainty, Cliff walking with traps featuring transition function uncertainty, Rover navigation task involving partial observability in obstacle locations.}
    \label{fig:casestudies}
\end{figure}

\begin{figure}[!ht]
  \centering
  \begin{subfigure}{0.4\textwidth}
    \includegraphics[width=\linewidth]{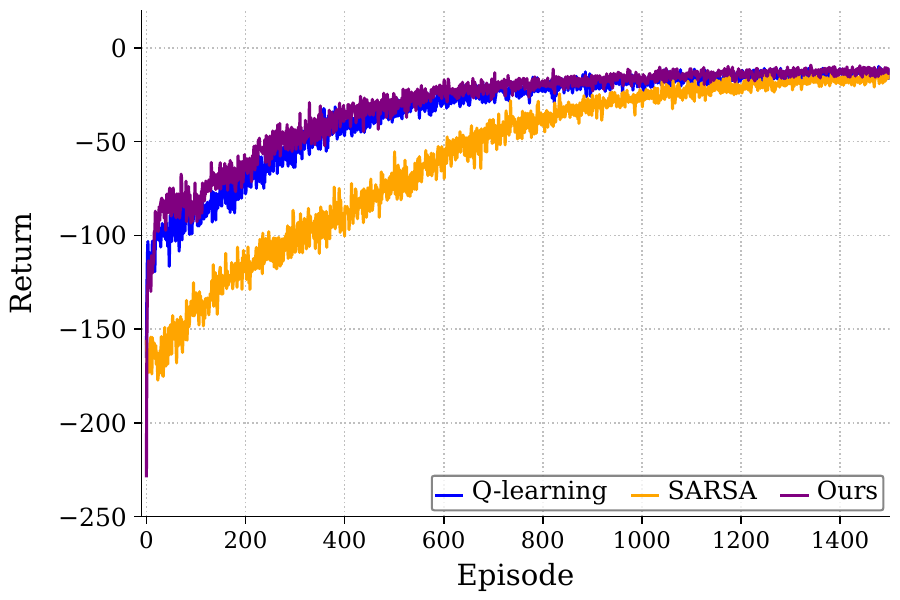}
    \caption{Grid-world}
    \label{fig:subfig-a}
  \end{subfigure}
  
  \begin{subfigure}{0.4\textwidth}
    \includegraphics[width=\linewidth]{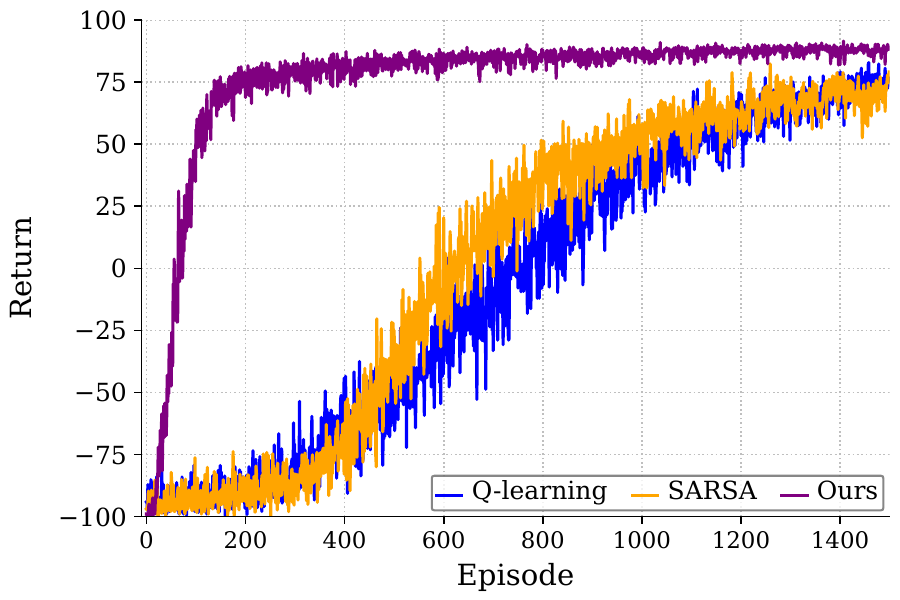}
    \caption{Cliff walking}
    \label{fig:subfig-b}
  \end{subfigure}
  
  \begin{subfigure}{0.4\textwidth}
    \includegraphics[width=\linewidth]{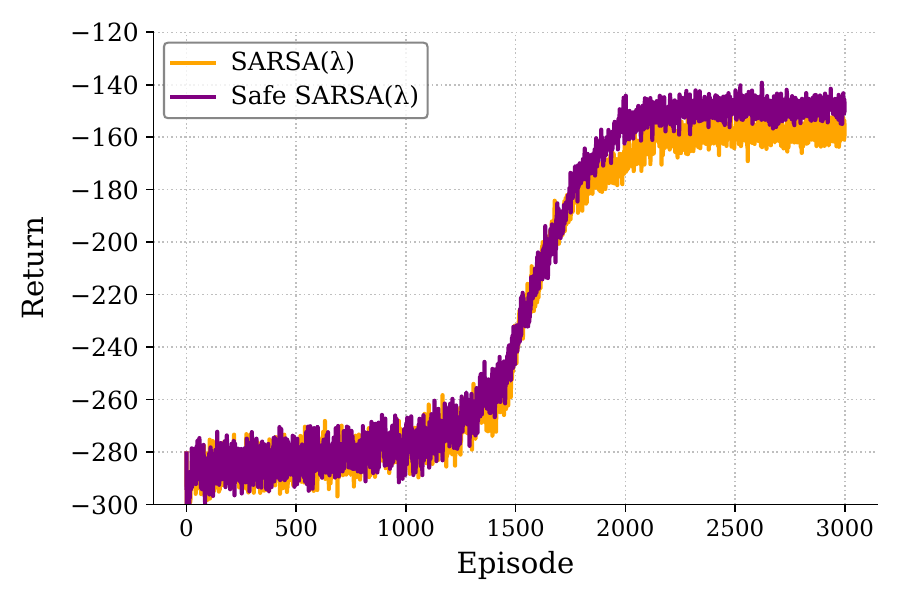}
    \caption{Rover navigation}
    \label{fig:subfig-c}
  \end{subfigure}
  \caption{Comparison between average cumulative reward over $50$ random seeds. In all case studies, our algorithm outperforms other baselines.}
  \label{returnLU}
\end{figure}

\subsection{CASE STUDY 1: GRID-WORLD WITH REWARD UNCERTAINTY}
We consider a $10 \times 10$ grid-world environment with normal, goal, and slippery states \cite{gehring2013smart}. The agent can move up, down, left, and right. For any movement to a normal state, the agent receives the reward of $-1$, while transitions to slippery states result in a random reward in the range $[-12,10]$. Collisions with walls incur a reward of $-10$. The episode terminates when the agent either reaches the goal state in the top-right corner or completes a maximum of $100$ steps.

The results in Fig. \ref{returnLU} demonstrate that the OT-guided SARSA algorithm converges to a higher return value and exhibits higher stability throughout the learning process, mainly because of the uncertainty score term, which guides exploration toward safer and more consistent actions. Furthermore, we present the state visitation map in Fig. \ref{mapGridLU}. As expected, SARSA algorithm demonstrates a high frequency of visits to slippery regions (darker red). In contrast, Q-learning performs better by exploring more efficient paths, but it still exhibits notable visits to unsafe states in comparison to our algorithm. 

\begin{figure*}[th!]
    \centering
     \includegraphics[scale=0.23]{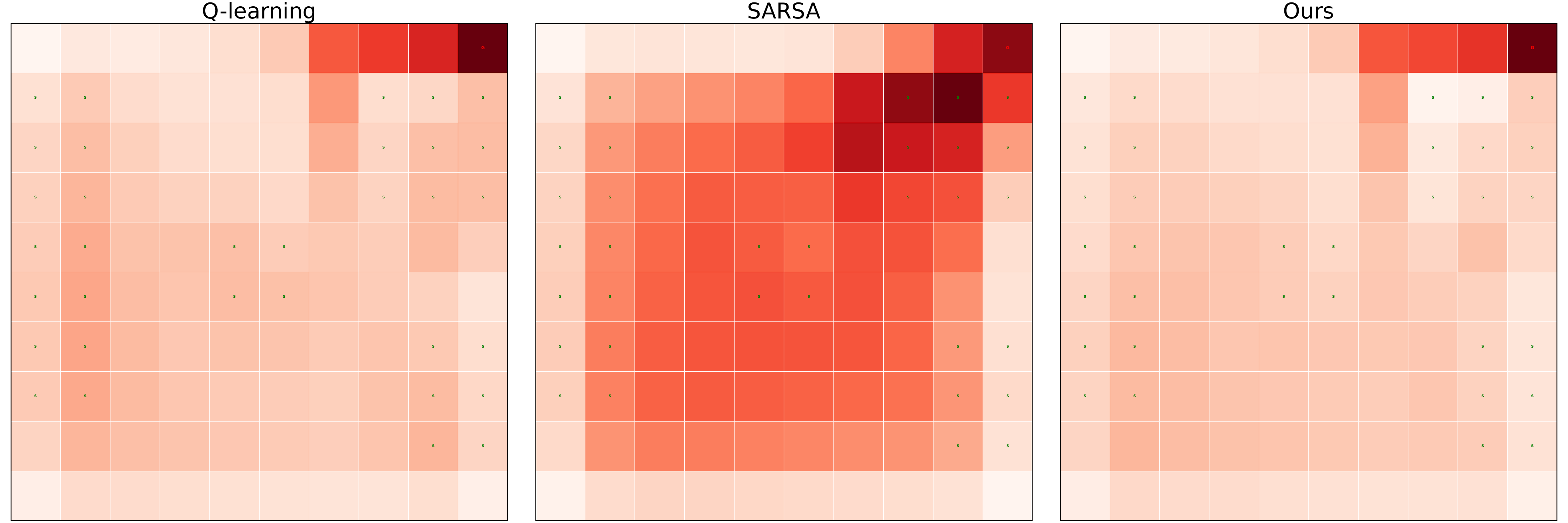}
     \caption{Comparison between state visitation density on grid-world case study. Our algorithm visitation to slippery states with reward uncertainty is less than both SARSA and Q-learning.}
     \label{mapGridLU}
\end{figure*}


\subsection{CASE STUDY 2: CLIFF WALKING WITH TRANSITION UNCERTAINTY}
This environment consists of three zones: the cliff, the trap, and the feasible regions. The agent starts at the bottom-left corner to reach the goal at the bottom-right corner while avoiding the cliff zone \cite{xuan2022sem}. The agent can move freely within the feasible region in four directions: up, down, left, and right. Entering the cliff region results in task failure. Entering the trap region forces the agent to move downward, regardless of its chosen action, eventually ending up in the cliff region. Each movement yields a reward of $-1$. If the agent collides with the environment borders, its position remains unchanged, but it still earns the movement reward. Reaching the target earns the agent a reward of $101$, while entering the cliff region results in a $-49$ penalty. Upon reaching either the target or the cliff region, the training episode restarts, and the agent is reset to its starting position.

\begin{figure*}[h!]
    \centering
     \includegraphics[scale=0.3]{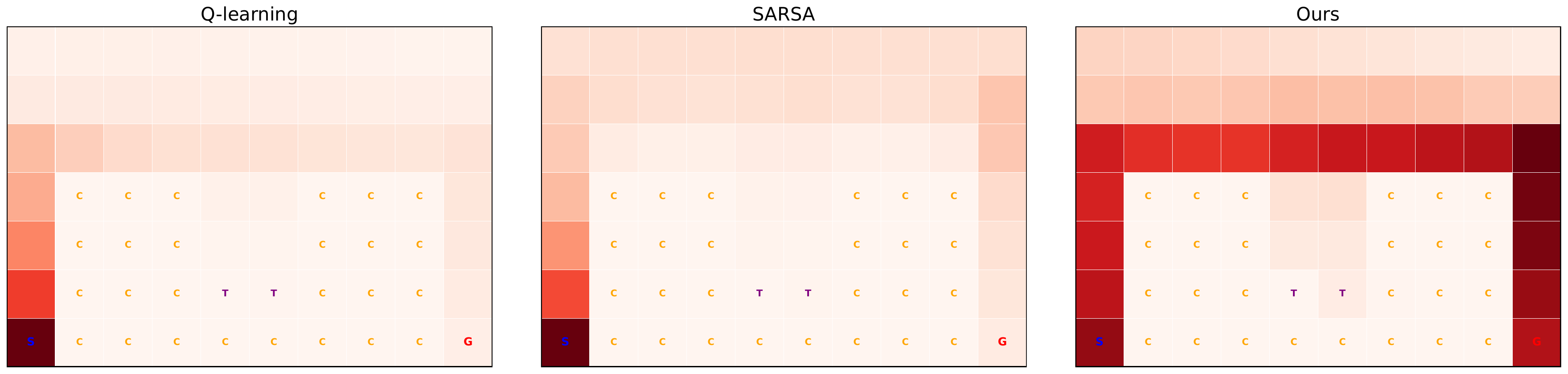}
     \caption{Comparison between the density of state visitation on the cliff walking case study. Our algorithm provides fewer visits to the cliff region than SARSA and Q-learning, and reaches the goal state more often (darker red).}
     \label{mapCliffLU}
\end{figure*}

As shown in Fig. \ref{returnLU} for this case study, our algorithm achieves rapid convergence, mainly because of prioritizing less uncertain actions. It also demonstrates a higher confidence in its return estimates. In contrast, both Q-learning and SARSA display greater variability, which reflects higher uncertainty in their returns. For this case study, the state visitation graph in Fig. \ref{mapCliffHU} highlights the limitations of SARSA, where the agent struggles to identify the optimal path to the goal state. Consequently, most episodes end without successfully reaching the goal. Q-learning performance is closer to our algorithm, however, the rate of reaching the goal and escaping from the cliff in the initial episodes is lower than ours.

\subsection{CASE STUDY 3: ROVER NAVIGATION WITH PARTIAL OBSERVABILITY}
In this case study, a rover must navigate a two-dimensional terrain map represented as a $10 \times 10$ grid, where $3$ of the grid cells are obstacles \cite{ahmadi2023risk}. Each grid cell represents a state, and the rover can move in eight geographic directions. However, the environment is stochastic; for example, as shown in Fig. \ref{fig:casestudies}, when the rover takes the action east, it moves to the intended grid cell with a probability of $0.9$ but may move to one of the adjacent cells with a probability of $0.05$. Partial observability exists because the rover cannot directly detect the locations of obstacle cells through its measurements. When the rover moves to a cell adjacent to an obstacle, it can identify the exact location of the obstacle (magenta) with a probability of $0.6$ and observe a probability distribution over nearby cells (pink). Colliding with an obstacle results in an immediate penalty of $10$, while reaching the goal region provides no immediate reward. All other grid cells impose a penalty of $2$.


As demonstrated in Fig. \ref{returnLU}, OT-guided SARSA($\lambda$) achieves superior performance compared to standard SARSA($\lambda$).

\begin{table*}[!t]
    \centering
    \caption{Ablation study on grid-world environment. Average return and std over different algorithms for low-uncertainty (LU), high-uncertainty (HU), and large state space scenarios.}
    \label{tab:grid}
    \renewcommand{\arraystretch}{2} 
    \begin{tabular}{lcccc}
        \toprule
        \textbf{Scenario} & \textbf{SARSA} & \textbf{Q-learning} & \textbf{Ours} \\
        \midrule
        $10 \times 10$ LU & $-15.70 \pm 17.32$ & $-13.38 \pm 14.53$ & \boldmath{\hl{$-12.04$} $\pm$ \hl{$12.83$}} \\
        $10 \times 10$ HU & $-16.26 \pm 20.38$ & $-14.14 \pm 18.52$ & \boldmath{\hl{$-12.92$} $\pm$ \hl{$16.79$}} \\
        $30 \times 30$ HU & $-119.16 \pm 126.65$ & $-108.21 \pm 118.76$ & \boldmath{\hl{$-98.77$} $\pm$ \hl{$112.15$}} \\ 
        \bottomrule
    \end{tabular}
    \renewcommand{\arraystretch}{1.0} 
\end{table*}

\subsection{ABLATION STUDY}
\noindent{\textbf{Performance under varying uncertainty and scale: }}We further analyze the performance of our algorithm under different levels of uncertainty and environment size. In particular, we consider the above settings as low uncertainty (LU). To model high uncertainty (HU) in the grid-world scenario, we increase the number of slippery states from $33$ to $50$, while for cliff walking, we enhance the number of traps from $2$ to $6$ to have more stochasticity in the environment. For the POMDP case study, we raised the number of partially observable obstacles from $3$ to $10$ to amplify the level of uncertainty. Additionally, to assess the impact of environment size on agent performance, we expand the grid-world and rover navigation tasks from $10 \times 10$ to $30 \times 30$, and the cliff walking from $7 \times 10$ to $21 \times 30$.


\begin{table*}[!t]
    \centering
    \caption{Ablation study on cliff walking case study. Average performance metrics for low-uncertainty (LU), high-uncertainty (HU), and large state space scenarios.}
    \label{tab:cliff}
    \renewcommand{\arraystretch}{2} 
    \begin{tabular}{lcccc}
        \toprule
        \textbf{Scenario} & \textbf{Return/Failures} & \textbf{SARSA} & \textbf{Q-learning} & \textbf{Ours} \\
        \midrule
        \multirow{2}{*}{$10 \times 7$ LU} 
        & return $\pm$ std & $72.48 \pm 36.1$ & $74.59 \pm 28.52$ & \boldmath{\hl{$87.99$} $\pm$ \hl{$10.65$}}  \\
        & failures & $72.74$ & $61.46$ & \boldmath{\hl{$21.74$}} \\
        \midrule
        \multirow{2}{*}{$10 \times 7$ HU} 
        & return $\pm$ std & $69.82 \pm 38.69$ & $84.27 \pm 23.99$ & \boldmath{\hl{$89.53$} $\pm$ \hl{$8.65$}}  \\
        & failures & $196.94$ & $77.56$ & \boldmath{\hl{$30.14$}}  \\
        \midrule
        \multirow{2}{*}{$30 \times 21$ HU}  
        & return $\pm$ std & $-177.07 \pm 54.61$ & $42.11 \pm 37.44$ & \boldmath{\hl{$55.59$} $\pm$ \hl{$22.36$}}  \\
        & failures & $523.64$ & $362.66$ & \boldmath{\hl{$196.48$}}  \\
        \bottomrule
    \end{tabular}
    \renewcommand{\arraystretch}{1.0} 
\end{table*}

\begin{figure*}[ht!]
    \centering
     \includegraphics[scale=0.23]{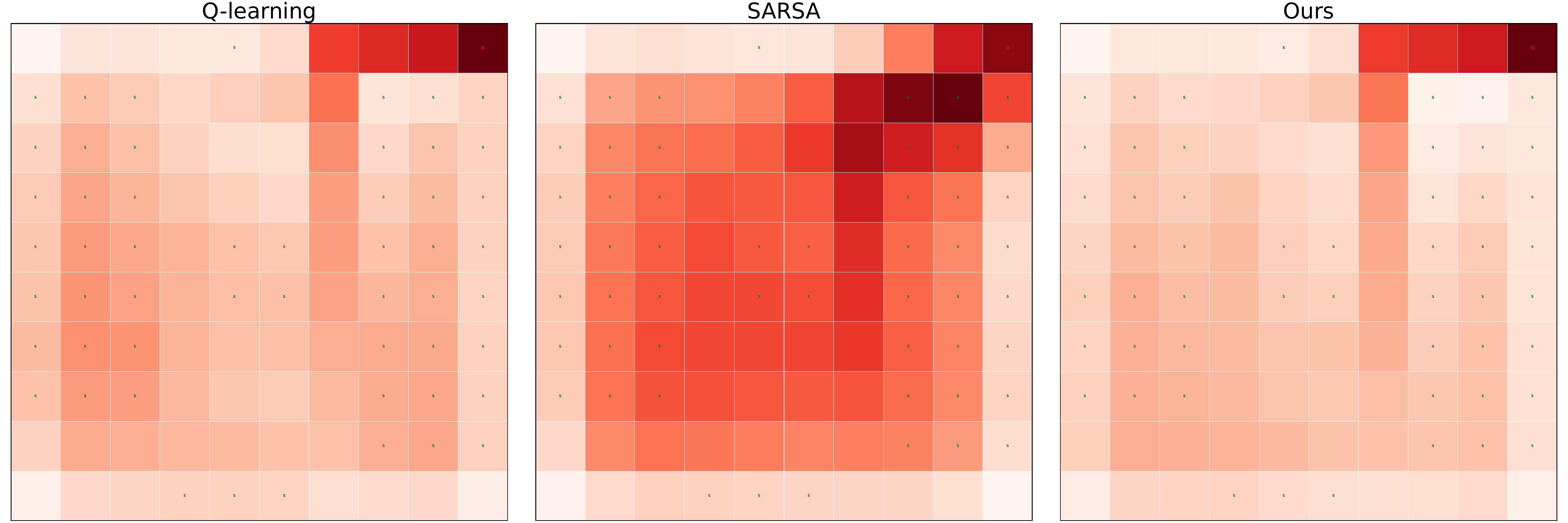}
     \caption{Comparison between state visitation density for different algorithms on grid-world with high uncertainty. Our algorithm visits slippery states less frequently than both baselines.}
     \label{mapGridHU}
\end{figure*}

\begin{figure*}[ht!]
    \centering
     \includegraphics[scale=0.29]{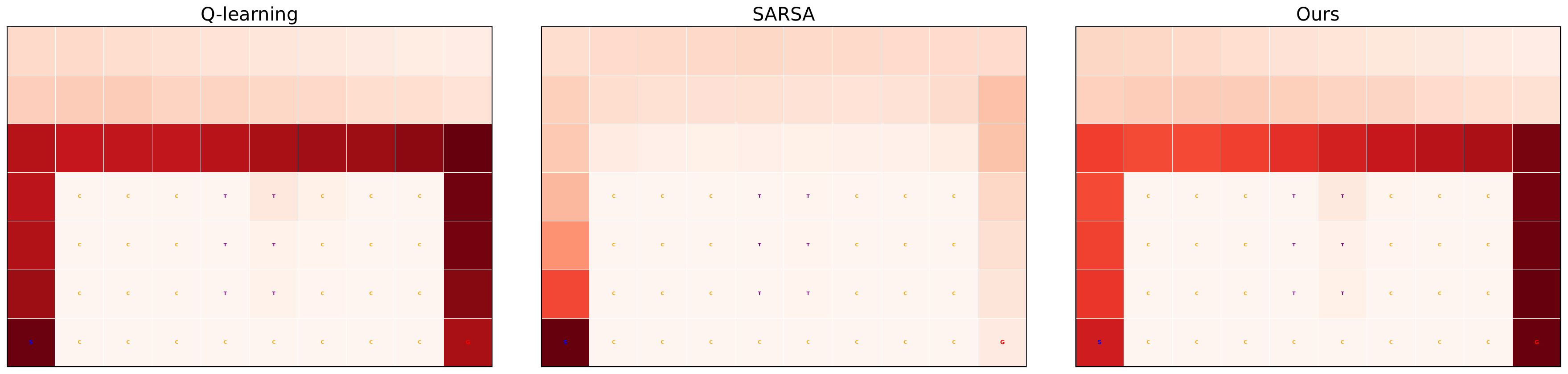}
     \caption{Comparison between the density of state visitation for different algorithms in the case of cliff walking with high uncertainty. While both our algorithm and Q-learning agents were able to find the optimal path to the goal state, our algorithm fell to the cliff states less often.}
     \label{mapCliffHU}
\end{figure*}


\begin{table*}[!t]
    \centering
    \caption{Ablation study on rover navigation task. Average performance metrics for low-uncertainty (LU), high-uncertainty (HU), and large state space scenarios.}
    \label{tab:pomdp}
    \renewcommand{\arraystretch}{2} 
    \begin{tabular}{lcccc}
        \toprule
        \textbf{Scenario} & \textbf{Return/Failures} & \textbf{SARSA($\lambda$)} & \textbf{Ours} \\
        \midrule
        \multirow{2}{*}{$10 \times 10$ LU} 
        & return $\pm$ std & $-155.04 \pm 109.09$ & \boldmath{\hl{$-149.96$} $\pm$ \hl{$106.41$}}  \\
        & failures & $2805.73$ & \boldmath{\hl{$2663.48$}} \\
        \midrule
        \multirow{2}{*}{$10 \times 10$ HU} 
        & return $\pm$ std & \boldmath{\hl{$-247.70$}} $\pm$ $102.58$ & $-264.00$ $\pm$ \boldmath{\hl{$89.72$}} \\
        & failures & $9039.29$ & \boldmath{\hl{$8948.18$}} \\
        \midrule
        \multirow{2}{*}{$30 \times 30$ HU}  
        & return $\pm$ std & $-1024.74 \pm 57.09$ & \boldmath{\hl{$-1020.70$} $\pm$ \hl{$39.26$}} \\
        & failures & $30725.38$ & \boldmath{\hl{$30026.71$}} \\
        \bottomrule
    \end{tabular}
    \renewcommand{\arraystretch}{1.0} 
\end{table*}

Table. \ref{tab:grid}, \ref{tab:cliff}, and \ref{tab:pomdp} present a quantitative comparison of OT-guided SARSA performance to other baselines, providing the average return and standard deviation (std) under LU, HU, and increasing environment size scenarios across last $20$ episodes for all case studies. The results in the Table. \ref{tab:grid}, and \ref{tab:cliff} confirm that OT-guided SARSA achieves the highest return and the lowest std in different scenarios. Moreover, in the cliff walking case study, the number of failures for our agent is significantly lower than both SARSA and Q-learning. For instance, in the LU scenario, OT-guided SARSA reduces failures by $35\%$ compared to SARSA and $30\%$ compared to Q-learning. This reduction is even greater in the other two scenarios. Overall, in both MDP case studies, our algorithm obtained a higher cumulative reward than Q-learning and SARSA, while improving the stability and the safety of the agent by avoiding unpredictable actions. Furthermore, for both MDP case studies, although increasing the size of the environment causes lower return values, OT-guided SARSA still maintains the best performance.

For the POMDP case study, as can be seen in the Table. \ref{tab:pomdp}, by increasing the partial observability degree (HU), the agent performance falls back. This shows that the performance of our algorithm can be influenced by the degree of partial observability in the environment. Specifically, when the agent receives indistinguishable or highly similar observations for different underlying states, the accuracy of the estimated Q-values, target distribution, and, consequently, the reliability of the uncertainty score becomes questionable. The number of failures for OT-guided SARSA($\lambda$) across the three scenarios is slightly lower than SARSA($\lambda$).

\noindent{\textbf{Sensitivity analysis of safety coefficient $\beta$: }}We additionally provide a sensitivity analysis on the hyperparameter $\beta$. Fig. \ref{Sensitivity} depicts the learning curve for different $\beta$ values for the cliff walking case study. As $\beta$ increases, the agent prioritizes safety more strongly, which reduces its willingness to explore potentially risky but rewarding paths. When $\beta$ is small, the uncertainty score has less impact on Q-values, which causes the algorithm to behave similarly to standard SARSA. This trade-off highlights the importance of carefully tuning $\beta$ to achieve an appropriate balance between safety and performance.

\begin{figure}[h!]
    \centering
     \includegraphics[scale=0.5]{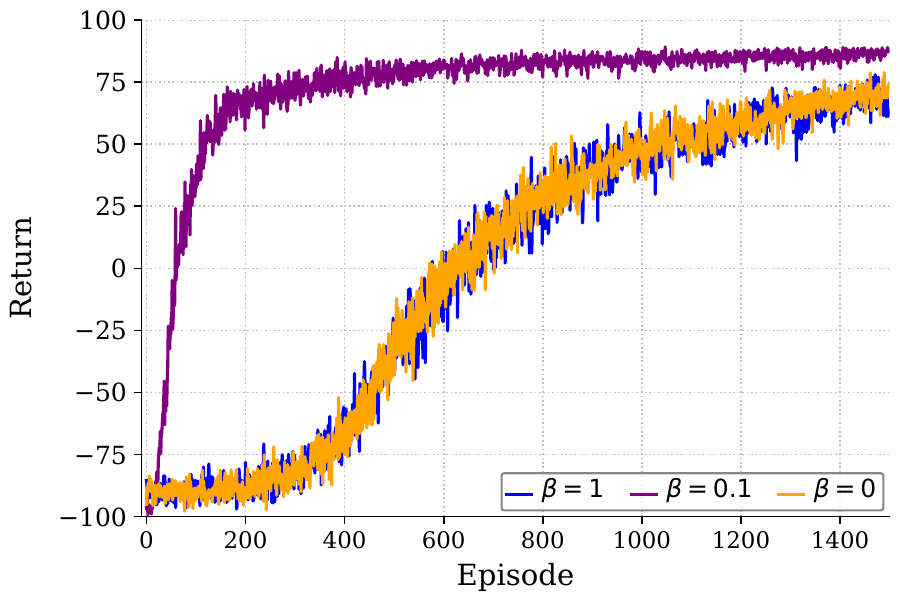}
     \caption{Sensitivity analysis of hyperparameter $\beta$ in cliff walking case study.}
     \label{Sensitivity}
\end{figure}

\section{CONCLUSION}
\label{sec: Conclusion}

We presented an uncertainty-oriented TD algorithm based on optimal transport theory. We showed the effectiveness of this approach in encouraging agent to prioritize less uncertain actions, leading to a reduction in visits to unsafe states and an improvement in cumulative rewards. Compared to standard TD algorithms, our algorithm demonstrated robust performance in environments with reward, transition, and state uncertainties. Although our algorithm outperforms standard TD learning methods, it has its limitations. Determining the optimal transport map for candidate actions at each state is computationally expensive. While using the entropy-regularized extension of OT reduces this computational cost, further improvements in computational efficiency will be a focus of future work. Another future direction is the extension of our algorithm to constrained MDPs in which the agent seeks to maximize the expected return while satisfying some constraints, typically expressed as bounds on expected costs. The optimal transport formulation could then be adapted by modifying its cost matrix to reflect these constraints. Finally, investigating our algorithm in the continuous domain by considering the distribution of action preferences derived from Q-values approximated as Gaussian is feasible and cost-efficient because Wasserstein distance provides a closed-form solution for Gaussian distributions. This would be an interesting direction to facilitate the integration of our approach into robotic control and autonomous navigation applications.



\renewcommand*{\bibfont}{\footnotesize}

\printbibliography

@article{geibel2005risk,
  title={Risk-sensitive reinforcement learning applied to control under constraints},
  author={Geibel, Peter and Wysotzki, Fritz},
  journal={Journal of Artificial Intelligence Research},
  volume={24},
  pages={81--108},
  year={2005}
}

@inproceedings{mavrin2019distributional,
  title={Distributional reinforcement learning for efficient exploration},
  author={Mavrin, Borislav and Yao, Hengshuai and Kong, Linglong and Wu, Kaiwen and Yu, Yaoliang},
  booktitle={International Conference on Machine Learning},
  pages={4424--4434},
  year={2019},
  organization={PMLR}
}

@article{yu2020mopo,
  title={M{OPO}: Model-based offline policy optimization},
  author={Yu, Tianhe and Thomas, Garrett and Yu, Lantao and Ermon, Stefano and Zou, James Y and Levine, Sergey and Finn, Chelsea and Ma, Tengyu},
  journal={Advances in Neural Information Processing Systems},
  volume={33},
  pages={14129--14142},
  year={2020}
}

@article{nikolov2018information,
  title={Information-directed exploration for deep reinforcement learning},
  author={Nikolov, Nikolay and Kirschner, Johannes and Berkenkamp, Felix and Krause, Andreas},
  journal={arXiv preprint arXiv:1812.07544},
  year={2018}
}

@inproceedings{chen2022efficient,
  title={Efficient and stable information directed exploration for continuous reinforcement learning},
  author={Chen, Mingzhe and Xiao, Xi and Zhang, Wanpeng and Gao, Xiaotian},
  booktitle={IEEE International Conference on Acoustics, Speech and Signal Processing (ICASSP)},
  pages={4023--4027},
  year={2022}
}

@article{tang2018exploration,
  title={Exploration by distributional reinforcement learning},
  author={Tang, Yunhao and Agrawal, Shipra},
  journal={arXiv preprint arXiv:1805.01907},
  year={2018}
}

@article{zhou2021non,
  title={Non-decreasing quantile function network with efficient exploration for distributional reinforcement learning},
  author={Zhou, Fan and Zhu, Zhoufan and Kuang, Qi and Zhang, Liwen},
  journal={arXiv preprint arXiv:2105.06696},
  year={2021}
}

@inproceedings{dabney2018distributional,
  title={Distributional reinforcement learning with quantile regression},
  author={Dabney, Will and Rowland, Mark and Bellemare, Marc and Munos, R{\'e}mi},
  booktitle={Proceedings of the AAAI conference on Artificial Intelligence},
  volume={32},
  number={1},
  year={2018}
}

@inproceedings{o2018uncertainty,
  title={The uncertainty {B}ellman equation and exploration},
  author={O’Donoghue, Brendan and Osband, Ian and Munos, Remi and Mnih, Volodymyr},
  booktitle={International Conference on Machine Learning},
  pages={3836--3845},
  year={2018}
}

@inproceedings{malekzadeh2023unified,
  title={A unified uncertainty-aware exploration: Combining epistemic and aleatory uncertainty},
  author={Malekzadeh, Parvin and Hou, Ming and Plataniotis, Konstantinos N},
  booktitle={IEEE International Conference on Acoustics, Speech and Signal Processing (ICASSP)},
  pages={1--5},
  year={2023}
}

@book{krishnamurthy2016partially,
  title={Partially observed Markov decision processes},
  author={Krishnamurthy, Vikram},
  year={2016},
  publisher={Cambridge University Press}
}

@article{nkhumise2025studying,
title={Studying Exploration in {RL}: An Optimal Transport Analysis of Occupancy Measure Trajectories},
author={Reabetswe M. Nkhumise and Debabrota Basu and Tony J. Prescott and Aditya Gilra},
journal={Transactions on Machine Learning Research},
issn={2835-8856},
year={2025},
url={https://openreview.net/forum?id=pdC092Nn8N},
note={}
}

@article{brunke2022safe,
  title={Safe learning in robotics: From learning-based control to safe reinforcement learning},
  author={Brunke, Lukas and Greeff, Melissa and Hall, Adam W and Yuan, Zhaocong and Zhou, Siqi and Panerati, Jacopo and Schoellig, Angela P},
  journal={Annual Review of Control, Robotics, and Autonomous Systems},
  volume={5},
  number={1},
  pages={411--444},
  year={2022},
  publisher={Annual Reviews}
}

@article{hambly2023recent,
  title={Recent advances in reinforcement learning in finance},
  author={Hambly, Ben and Xu, Renyuan and Yang, Huining},
  journal={Mathematical Finance},
  volume={33},
  number={3},
  pages={437--503},
  year={2023},
  publisher={Wiley Online Library}
}

@inproceedings{kendall2019learning,
  title={Learning to drive in a day},
  author={Kendall, Alex and Hawke, Jeffrey and Janz, David and Mazur, Przemyslaw and Reda, Daniele and Allen, John-Mark and Lam, Vinh-Dieu and Bewley, Alex and Shah, Amar},
  booktitle={International Conference on Robotics and Automation (ICRA)},
  pages={8248--8254},
  year={2019},
  organization={IEEE}
}

@article{richter2019open,
  title={Open-sourced reinforcement learning environments for surgical robotics},
  author={Richter, Florian and Orosco, Ryan K and Yip, Michael C},
  journal={arXiv preprint arXiv:1903.02090},
  year={2019}
}

@article{ahmadi2023risk,
  title={Risk-averse decision making under uncertainty},
  author={Ahmadi, Mohamadreza and Rosolia, Ugo and Ingham, Michel D and Murray, Richard M and Ames, Aaron D},
  journal={IEEE Transactions on Automatic Control},
  volume={69},
  number={1},
  pages={55--68},
  year={2023},
  publisher={IEEE}
}

@article{cuturi2013sinkhorn,
  title={Sinkhorn distances: Lightspeed computation of optimal transport},
  author={Cuturi, Marco},
  journal={Advances in Neural Information Processing Systems},
  volume={26},
  year={2013}
}

@inproceedings{okawa2022safe,
  title={Safe exploration method for reinforcement learning under existence of disturbance},
  author={Okawa, Yoshihiro and Sasaki, Tomotake and Yanami, Hitoshi and Namerikawa, Toru},
  booktitle={Joint European Conference on Machine Learning and Knowledge Discovery in Databases},
  pages={132--147},
  year={2022},
  organization={Springer}
}

@inproceedings{gehring2013smart,
  title={Smart exploration in reinforcement learning using absolute temporal difference errors},
  author={Gehring, Clement and Precup, Doina},
  booktitle={International Conference on Autonomous Agents and Multi-agent Systems},
  pages={1037--1044},
  year={2013}
}

@article{gaskett2003reinforcement,
  title={Reinforcement learning under circumstances beyond its control},
  author={Gaskett, Chris},
  Journal={International Conference on Computational Intelligence for Modelling Control and Automation},
  year={2003}
}

@article{shahrooei2024optimal,
  title={Optimal Transport-Assisted Risk-Sensitive {Q}-Learning},
  author={Shahrooei, Zahra and Baheri, Ali},
  journal={arXiv preprint arXiv:2406.11774},
  year={2024}
}

@article{chow2019lyapunov,
  title={Lyapunov-based safe policy optimization for continuous control},
  author={Chow, Yinlam and Nachum, Ofir and Faust, Aleksandra and Duenez-Guzman, Edgar and Ghavamzadeh, Mohammad},
  journal={arXiv preprint arXiv:1901.10031},
  year={2019}
}

@article{sato2001td,
  title={{TD} algorithm for the variance of return and mean-variance reinforcement learning},
  author={Sato, Makoto and Kimura, Hajime and Kobayashi, Shibenobu},
  journal={Transactions of the Japanese Society for Artificial Intelligence},
  volume={16},
  number={3},
  pages={353--362},
  year={2001},
  publisher={Japanese Society for Artificial Intelligence}
}

@incollection{heger1994consideration,
  title={Consideration of risk in reinforcement learning},
  author={Heger, Matthias},
  booktitle={Machine Learning Proceedings},
  pages={105--111},
  year={1994},
  publisher={Elsevier}
}

@article{law2005risk,
  title={Risk-directed exploration in reinforcement learning},
  author={Law, Edith LM},
  Journal={PhD thesis},
  publisher={McGill University},
  year={2005}
}

@article{xuan2022sem,
  title={S{E}{M}: Safe exploration mask for {Q}-learning},
  author={Xuan, Chengbin and Zhang, Feng and Lam, Hak-Keung},
  journal={Engineering Applications of Artificial Intelligence},
  volume={111},
  pages={104765},
  year={2022},
  publisher={Elsevier}
}

@article{metelli2019propagating,
  title={Propagating uncertainty in reinforcement learning via {W}asserstein barycenters},
  author={Metelli, Alberto Maria and Likmeta, Amarildo and Restelli, Marcello},
  journal={Advances in Neural Information Processing Systems},
  volume={32},
  year={2019}
}

@inproceedings{achiam2017constrained,
  title={Constrained policy optimization},
  author={Achiam, Joshua and Held, David and Tamar, Aviv and Abbeel, Pieter},
  booktitle={International Conference on {M}achine {L}earning},
  pages={22--31},
  year={2017},
  organization={PMLR}
}

@article{watkins1992q,
  title={Q-learning},
  author={Watkins, Christopher JCH and Dayan, Peter},
  journal={Machine Learning},
  volume={8},
  pages={279--292},
  year={1992},
  publisher={Springer}
}

@book{santambrogio2015optimal,
  title={Optimal transport for applied mathematicians},
  author={Santambrogio, Filippo},
  volume={87},
  year={2015},
  publisher={Springer}
}

@book{villani2009optimal,
  title={Optimal transport: old and new},
  author={Villani, C{\'e}dric},
  volume={338},
  year={2009},
  publisher={Springer}
}

@book{sutton1998reinforcement,
  title={Reinforcement learning: An introduction},
  author={Sutton, Richard S and Barto, Andrew G},
  volume={1},
  number={1},
  year={1998},
  publisher={MIT press Cambridge}
}

@inproceedings{baheri2023understanding,
title={Understanding Reward Ambiguity Through Optimal Transport Theory in Inverse Reinforcement Learning},
author={Ali Baheri},
booktitle={NeurIPS Workshop Optimal Transport and Machine Learning},
year={2023},
url={https://openreview.net/forum?id=L7c0hEWO2m}
}

@article{baheri2025wasserstein,
  title={Wasserstein Adaptive Value Estimation for Actor-Critic Reinforcement Learning},
  author={Baheri, Ali and Shahrooei, Zahra and Salgarkar, Chirayu},
  journal={arXiv preprint arXiv:2501.10605},
  year={2025}
}

@article{dai2023safe,
  title={Safe {RLHF}: Safe reinforcement learning from human feedback},
  author={Dai, Josef and Pan, Xuehai and Sun, Ruiyang and Ji, Jiaming and Xu, Xinbo and Liu, Mickel and Wang, Yizhou and Yang, Yaodong},
  journal={arXiv preprint arXiv:2310.12773},
  year={2023}
}

@article{baheri2023risk,
  title={Risk-aware reinforcement learning through optimal transport theory},
  author={Baheri, Ali},
  journal={arXiv preprint arXiv:2309.06239},
  year={2023}
}

@article{gu2024review,
  title={A review of safe reinforcement learning: Methods, theories and applications},
  author={Gu, Shangding and Yang, Long and Du, Yali and Chen, Guang and Walter, Florian and Wang, Jun and Knoll, Alois},
  journal={IEEE Transactions on Pattern Analysis and Machine Intelligence},
  year={2024},
  publisher={IEEE}
}

@article{queeney2023optimal,
  title={Optimal transport perturbations for safe reinforcement learning with robustness guarantees},
  author={Queeney, James and Ozcan, Erhan Can and Paschalidis, Ioannis Ch and Cassandras, Christos G},
  journal={arXiv preprint arXiv:2301.13375},
  year={2023}
}

@article{garcia2015comprehensive,
  title={A comprehensive survey on safe reinforcement learning},
  author={Garc{\i}a, Javier and Fern{\'a}ndez, Fernando},
  journal={Journal of Machine Learning Research},
  volume={16},
  number={1},
  pages={1437--1480},
  year={2015}
}







\end{document}